\documentclass{article}
\usepackage{graphicx} 
\usepackage{mlncp_2023}
\usepackage{amsthm}
\usepackage{amsmath}
\usepackage{amssymb}
\usepackage{placeins}
\usepackage{xcolor}
\newtheorem{lemma}{Lemma}
\newtheorem{proposition}{Proposition}
\newtheorem{theorem}{Theorem}

\usepackage{caption}
\usepackage[affil-it]{authblk}
\usepackage{multicol} 
\newcommand{\extrasmall}{\fontsize{8pt}{9pt}\selectfont}

\title{Bibliography management: \texttt{natbib} package}

\title{Bayesian Metaplasticity from Synaptic Uncertainty}

\author{
    \extrasmall{\textbf{Djohan Bonnet\textsuperscript{1a,2}}}, 
    \extrasmall{\textbf{Tifenn Hirtzlin\textsuperscript{1a}}}, 
    \extrasmall{\textbf{Tarcisius Januel\textsuperscript{1a}}}, 
    \extrasmall{\textbf{Thomas Dalgaty\textsuperscript{1b}}}, 
    \extrasmall{\textbf{Damien Querlioz\textsuperscript{2}}}, 
    \extrasmall{\textbf{Elisa Vianello\textsuperscript{1a}}}
}

\affil{
    \extrasmall \textsuperscript{1}Université Grenoble Alpes, CEA, LETI\textsuperscript{ a}, LIST\textsuperscript{ b}, Grenoble, France\\
    \textsuperscript{2}Université Paris-Saclay, CNRS, C2N, Palaiseau, France
    \\ djohan.bonnet@cea.fr, damien.querlioz@c2n.upsaclay,fr, elisa.vianello@cea.fr
}

\begin{document}
\maketitle




\begin{abstract} 
Catastrophic forgetting remains a challenge for neural networks, especially in lifelong learning scenarios. In this study, we introduce MEtaplasticity from Synaptic Uncertainty (MESU), inspired by metaplasticity and Bayesian inference principles. MESU harnesses synaptic uncertainty to retain information over time, with its update rule closely approximating the diagonal Newton's method for synaptic updates. Through continual learning experiments on permuted MNIST tasks, we demonstrate MESU's remarkable capability to maintain learning performance across 100 tasks without the need of explicit task boundaries.
\end{abstract}

\section{Introduction}

Artificial Neural Networks (ANNs) have revolutionized various domains, but they suffer from catastrophic forgetting \citep{goodfellow2013empirical} when trained on multiple tasks sequentially, limiting their applicability in lifelong learning scenarios. Many previous approaches to address this issue have focused on weight consolidation techniques \citep{kirkpatrick2017overcoming,zenke2017continual,aljundi2018memory}. These methods identify important weights for previously learned tasks, typically using an approximation of the model's Hessian matrix. Simply, the Hessian diagonal component of each weight is used to protect it from further updates.
While effective, these methods normally require clear task boundaries, which diminishes their real-world utility. In an embedded learning context, a neural network will continuously encounter new training data in an online fashion with no means of determining boundaries between tasks.

The human brain avoids catastrophic forgetting without task boundaries through mechanisms still not fully understood, but believed to involve synaptic metaplasticity \citep{fusi2005cascade}, a form of synaptic consolidation. Synapses adapt their learning rate continuously,  not at the end of tasks. In ANNs, this type of metaplasticity has been shown empirically to reduce catastrophic forgetting \citep{laborieux2021synaptic,d2023synaptic}, but a theoretical understanding of these results is lacking. In parallel, recent work has also suggested that synaptic plasticity can be interpreted as Bayesian inference, where the synapse incorporates not just the synaptic weight but also a measure of its uncertainty \citep{aitchison2021synaptic}. This raises intriguing questions: Could the concept of weight uncertainty value in this work be analogous to synaptic local learning rate in metaplasticity? If so, could a Bayesian approach be the key to enabling lifelong learning without catastrophic forgetting, mimicking the capabilities of the human brain? 

In this paper, we propose a novel synaptic update rule based on Hessian-modulated learning and inspired by the principles of metaplasticity and Bayesian inference. We apply it to Bayesian neural networks \citep{blundell2015weight}, which are ANNs incorporating uncertainty about the weights, reminiscent of the biological model of \citep{aitchison2021synaptic}.  Our experiments demonstrate that despite its fully spatially and temporally local nature, our metaplasticity rule matches the computation of the Hessian matrix. We also show experimentally its high performance in mitigating catastrophic forgetting, without the need for explicit task boundaries.

An existing method also leverages Bayesian neural networks for continual learning without task boundaries \citep{zeno2018task}, but it exhibits a marked deterioration in its ability to learn new tasks as more and more tasks are encountered. Our method, which shines new light on this work, in contrast maintains learning ability perpetually, achieving an operation that more closely resembles the capabilities of the human brain.

 Our main contributions are as follows: 
\begin{itemize}
    \item  The synaptic uncertainty in Bayesian Neural Networks can serve as a metaplasticity parameter, enabling continual learning without task boundaries, and which keeps on learning after a high number of tasks, aligning with the biological insights of \cite{aitchison2021synaptic}.
    \item Theoretically and experimentally, we demonstrate that the associated update rule allows for a precise approximation of the Hessian diagonal, providing synaptic updates equivalent to the diagonal approximation of Newton's method.
    \item Our algorithm matches the state-of-the-art of learning permuted MNIST tasks without knowing task boundaries, while maintaining the capacity to learn even after encountering 100 tasks. 
\end{itemize}

\section{Theoretical Results}\label{sec3}

In Bayesian neural networks \citep{blundell2015weight,gal2016uncertainty,abdar2021review}, weights are typically represented as probability distributions defined by a mean-field Gaussian, $q(\omega|\theta)$, where $\theta$ encompasses both the mean ($\mu$) and standard deviation ($\sigma)$ values. Weight samples are expressed as $\omega=\mu+\epsilon \cdot \sigma$, with $\epsilon \sim \mathcal{N}(0, 1)$. 

During the training phase, each weight is learned using the backpropagation algorithm. To calculate the gradients of the expectation of the negative log-likelihood, we employ the Bayes by Backprop method. Consequently, the cost function is defined as $\mathcal{C}=\mathbb{E}_{q(\omega|\theta)}[\mathcal{L}(\omega)]$, where  $\mathcal{L}$ is the negative log-likelihood, and  the gradients are computed as

\begin{equation}\label{eq:Grads}
\frac{\partial \mathcal{C}}{\partial \mu} = \mathbb{E}_{\epsilon}\left[\frac{\partial \mathcal{L}(\omega)}{\partial \omega}\right]
\hspace{1.5cm}
 \frac{\partial \mathcal{C}}{\partial \sigma} = \mathbb{E}_{\epsilon}\left[\frac{\partial \mathcal{L}(\omega)}{\partial \omega} \times \epsilon\right]. 
\end{equation}

Using a learning rate $\gamma$,  $\sigma$ and $\mu$ at the update iteration $n+1$ are given by

\begin{equation}\label{eq:BBB}
\mu_{n+1}=\mu_n-\gamma \times \frac{\partial \mathcal{C}}{\partial \mu_n}
\hspace{1.5cm}
\sigma_{n+1}=\sigma_n-\gamma \times \frac{\partial \mathcal{C}}{\partial \sigma_n}.
\end{equation}

Like standard ANNs trained by the backpropagation algorithm, Bayesian neural networks trained by Bayes by Backprop are prone to catastrophic forgetting. 
To address the challenge of continual learning in Bayesian neural networks, our goal is to construct a metaplasticity function based on the continual computation of the Hessian diagonal of the negative log-likelihood. This Hessian diagonal serves to quantify the ``importance'' associated with each synapse. 
Synapses of greater importance exhibit a large contribution to the Hessian diagonal, requiring a reduction of their local learning rate.  However, while the utility of the Hessian is clear, it is costly to evaluate, and its continual evaluation is usually not feasible. Following the biology-inspired principle of metaplasticity \citep{fusi2005cascade,laborieux2021synaptic,jedlicka2022contributions}, in our approach, the Hessian diagonal is computed naturally using simple equations.

To understand our approach, we first highlight a connection between synaptic standard deviation ($\sigma$) and the Hessian diagonal.
We demonstrate Lemma 1 in Appendix A, which shows that when a component of the Hessian matrix has a high positive value, it results in a significant gradient with respect to the corresponding $\sigma$, leading to a decrease in $\sigma$ according to Eq.~(\ref{eq:BBB}). This observation suggests that synapses with greater importance for the previously learned tasks tend to have smaller $\sigma$ values.

\begin{lemma}[Hessian diagonal via First-Order Derivative in Bayesian Neural Networks]

Consider a mean-field Gaussian, $q(\omega|\theta)$, describing a Bayesian neural network, where $\theta=(\mu,\sigma)$, and weight samples are defined as $\omega=\mu+\epsilon \cdot \sigma$, with $\epsilon \sim \mathcal{N}(0, 1)$. Let $\mathcal{L}$ be the negative log-likelihood and $\mathcal{C}$ the expectation of $\mathcal{L}$, defined as $\mathcal{C}=\mathbb{E}_{q(\omega|\theta)}[\mathcal{L}(\omega)]$. The diagonal of the Hessian matrix of the expectation with respect to $\mu$ is given by: $H_D(\mu)=\frac{1}{\sigma} \frac{\partial \mathcal{C}}{\partial \sigma}$.
\end{lemma}

We now propose to modify the updates on both $\sigma$ and $\mu$ (Eq.~(\ref{eq:BBB}))  to achieve an update rule that we will show to be equivalent to the diagonal approximation of Newton's method \citep{becker1988improving}.

\begin{equation}\label{eq:Mesu}
\boxed{\mu_{n+1}=\mu_n-\sigma_n^2 \times \frac{\partial \mathcal{C}}{\partial \mu_n}
\hspace{1.5cm}
\sigma_{n+1}=\sigma_n-\sigma_n^2 \times \frac{\partial \mathcal{C}}{\partial \sigma_n}+ \frac{\sigma_n(\sigma_{prior}^2 -\sigma_n^2)}{\sigma_{res}^2}}.
\end{equation}

where $\sigma^2_n$ is the metaplasticity function, $\sigma_{prior}$ and $\sigma_{res}$ are two positive hyperparameters, with ${\sigma_{prior}^2} \ll {\sigma_{res}^2}$. $\sigma_n$ is attracted towards $\sigma_{prior}$, which can therefore be interpreted as a prior belief over $\sigma_n$ \citep{mackay1992practical}. $\sigma_{res}$ regulates the strength of the link between $\sigma_n$ and $\sigma_{prior}$.

We call these update rules ``MEtaplasticity from Synaptic Uncertainty'' (MESU).
When learning a new task $n+1$, MESU preserves the performance of the previously learned task $n$ by constraining the mean-field Gaussian $q(\omega|\theta_{n+1})$ to remain within a region of overlap with $q(\omega|\theta_{n})$, as illustrated in Fig.~\ref{fig:update sigma}. This constraint is implemented through the metaplasticity function $\sigma^2_n$ and aligns with parallel observations from computational neuroscience of \cite{aitchison2021synaptic}.

\begin{figure}[h]
  \centering
  \includegraphics[width=0.95\textwidth]{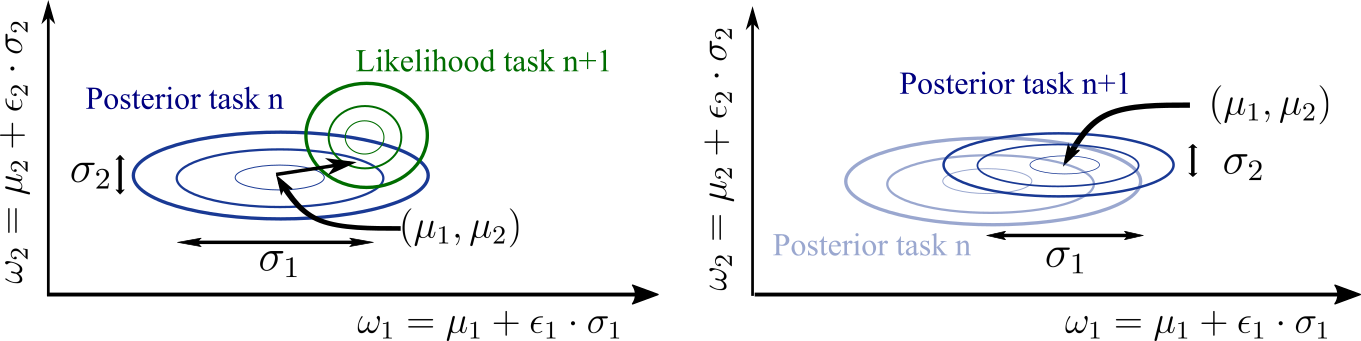}
  \caption{\textbf{Qualitative illustration of the update rule.} Contour plot of the density of the mean-field Gaussians $q(\omega|\theta_{n})$ and $q(\omega|\theta_{n+1})$ for the new task $n+1$ and the previously learned task $n$ (blue). Contour plot of the Likelihood of the incoming task number $n+1$ (green).}
  \label{fig:update sigma}
\end{figure}

\begin{paragraph}{The Locally Constant Curvature}
To justify MESU, we make a necessary assumption that, locally, the curvature of the loss function is constant and positive. This implies that for each sample $\omega$, $\frac{\partial^2 \mathcal{L}(\omega)}{\partial \omega^2}$ is a positive constant, which we call $\frac{1}{\sigma_L^2}$.
This assumption leads to two important consequences. First,  the Hessian diagonal with respect to $\sigma$ is equivalent to the Hessian diagonal with respect to $\mu$:

\begin{equation}\label{eq:Hessian_mu_constante}
H_D(\mu)=\mathbb{E}_{\epsilon}\left[\frac{\partial^2 \mathcal{L}(\omega)}{\partial \omega^2}\right]=\frac{1}{\sigma_L^2}
\end{equation}

\begin{equation}\label{eq:Hessian_sigma_constante}
H_D(\sigma)=\mathbb{E}_{\epsilon}\left[\frac{\partial^2 \mathcal{L}(\omega)}{\partial \omega^2}\right] \times \epsilon^2=\frac{\partial^2 \mathcal{L}(\omega)}{\partial \omega^2} \times \mathbb{E}_{\epsilon}[\epsilon^2]=\frac{1}{\sigma_L^2}.
\end{equation}

Second, the update over $\sigma$ becomes a recurrent sequence that can be analyzed. Using the lemma, Eq.~(\ref{eq:Mesu}), and Eq.~(\ref{eq:Hessian_mu_constante}), we obtain:

\begin{equation}\label{eq:sequence}
\sigma_{n+1}=\sigma_n(1-\frac{\sigma_n^2}{\sigma_L^2})+ \frac{\sigma_n(\sigma_{prior}^2 -\sigma_n^2)}{\sigma_{res}^2}.
\end{equation}

\end{paragraph}

A convergence analysis of  the sequence result leads to our theorem (see Appendix C for  full proof). 

\begin{theorem}\label{thm:MESU}[Hessian Diagonal via Synaptic Uncertainty in Bayesian Neural Networks]
Let $q(\omega|\theta)$ represent the mean-field Gaussian describing the Bayesian neural network, where $\theta=(\mu,\sigma)$, and weight samples are defined as $\omega=\mu+\epsilon \cdot \sigma$, with $\epsilon \sim \mathcal{N}(0, 1)$. Consider $\mathcal{L}$ as the negative log-likelihood and $\mathcal{C}$ as the expectation of $\mathcal{L}$, defined as $\mathcal{C}=\mathbb{E}_{q(\omega|\theta)}[\mathcal{L}(\omega)]$. Let $\sigma_{n+1}=\sigma_n-\sigma_n^2 \times \frac{\partial \mathcal{C}}{\partial \sigma_n}+ \frac{\sigma_n(\sigma_{prior}^2 -\sigma_n^2)}{\sigma_{res}^2}$ be the update rule over $\sigma$. If the loss curvature can be expressed as $\frac{1}{\sigma_L^2}$,  $\sigma_0<\sigma_L$ and, $\frac{\sigma_{prior}^2} {\sigma_{res}^2} \ll 1$, then the diagonal elements of the Hessian matrix with respect to $\mu$ and $\sigma$ are given by:
\begin{equation}\label{eq:HessianResult}
H_D(\mu)=H_D(\sigma)=\lim_{n \to +\infty} \left[\frac{\sigma_{prior}^2 }{\sigma_{res}^2\sigma_n^2}-\frac{1}{\sigma_{res}^2}\right].
\end{equation}
\end{theorem}

According to Theorem 1, for sufficiently large $n$,  the metaplasticity function $\sigma^2_n$ can be rewritten as:

\begin{equation}\label{eq:HessianResultMesu}
\boxed{\lim_{n \to +\infty}\sigma^2_n=\frac{\sigma_{prior}^2}{\sigma_{res}^2} \times \frac{1}{H_D(\mu) +\frac{1}{\sigma_{res}^2}}}.
\end{equation}

By substituting Eq.~(\ref{eq:HessianResultMesu}) into Eq.~(\ref{eq:Mesu}), the update on $\mu$ becomes proportional to the inverse of the Hessian diagonal. MESU is therefore equivalent to the diagonal approximation of Newton's method with a learning rate of $\frac{\sigma_{\text{prior}}^2}{\sigma_{\text{res}}^2}$ and a residual of $\frac{1}{\sigma_{res}^2}$. This constitutes the main theoretical results of the paper.
The residual term provides a safeguard against a null Hessian diagonal component.

\begin{paragraph}{Importance of the regularization term}\label{sec5}
The term $\frac{\sigma_n(\sigma_{\text{prior}}^2 - \sigma_n^2)}{\sigma_{\text{res}}^2}$ introduced in the $\sigma_{n+1}$ update in Eq.~(\ref{eq:Mesu}) constitutes the primary difference between MESU and the Bayesian Gradient
Descent (BGD) method of  \cite{zeno2018task}. 
Under the condition  $\frac{\partial \mathcal{C}}{\partial \sigma_n} \times \sigma_n \ll 1$ and with a scaling factor of 0.5 to the $\sigma$ update, Eq.~(\ref{eq:Mesu}) aligns with BGD when this term is omitted. 
This addition is crucial for two reasons. First, it prevents standard deviations from reaching zero, as shown in Eq.~(\ref{eq:fmeta2}). Second, it prevents large changes in the mean value while maintaining a constant standard deviation. Similarly, the synaptic metaplasticity model introduced by \cite{fusi2005cascade} inhibits any synaptic changes in value while maintaining a constant metaplastic value.
To further explore the relationship between MESU and BGD, in Appendix C, Theorem 2, we express the MESU update for $\sigma$ without the regularization term as a recurrent sequence and show that:

\begin{equation}\label{eq:HessianResult2}
H_D(\mu)=H_D(\sigma)=\lim_{n \to +\infty} \frac{1}{2n\sigma_n^2}.
\end{equation}

This result allows us to derive the analog of a metaplasticity function for the BGD method:

\begin{equation}\label{eq:fmeta2}
{\sigma_n^2= \frac{1}{2n H_D(\mu)}
\hspace{1.5cm}
\lim_{n \to +\infty}\frac{1}{2n H_D(\mu)}=0}.
\end{equation}

Consequently, the BGD method is also equivalent to the Newton method, but,  in that case, when the number of iterations $n$ is sufficiently large, the metaplasticity function, and therefore, the update on $\mu$, approach zero. This results in a linear decrease in plasticity with increasing number of iterations. In the following section, we 
realize that this difference between MESU and BGD has important consequences when learning a large number of tasks.

\end{paragraph}

\section{Experimental Results}\label{sec3}

We now evaluate our proposed MESU update on a Bayesian neural network  trained on MNIST and  its suitability for continual learning through the permuted MNIST task \citep{lecun1998mnist,goodfellow2013empirical}. We compare MESU to other approaches which rely on Hessian matrix approximation, and to BGD \citep{zeno2018task}. Appendix D provides full implementation details.

\begin{paragraph}{Validation of theoretical results}
We start by presenting various empirical results on the MNIST dataset. We employ a small fully-connected neural network with a single hidden layer of 100 units trained with MESU (Eq.~(\ref{eq:Mesu})). To quantify the capability of the MESU trained model to approximate the Hessian, we compute the true Hessian diagonal with respect to both $\sigma$ and $\mu$ using the PyTorch autograd function, based on a snapshot of the model synapses after 100,000 updates. First, we investigate our critical assumption of locally constant curvature (Eq.(\ref{eq:Hessian_mu_constante})), which predicted that the Hessian of the mean and standard deviation would be equivalent. In Fig.~\ref{fig:hessian_approx}a, we plot, as a point per synapse, the true Hessian diagonal with respect to $\sigma$ and $\mu$ where we observe that they are indeed equivalent.  Fig.~\ref{fig:hessian_approx}b shows that the Hessian diagonal with respect to $\mu$ is accurately approximated by Eq.~(\ref{eq:HessianResult}). This result substantiates Theorem 1, the most important theoretical result of our work and shows that  MESU is equivalent to Newton’s method.
Fig.~\ref{fig:hessian_approx}c  shows empirical support for Theorem 2:  a Bayesian neural network trained using BGD  is also equivalent to Newton's method to a large extent, but only when Hessian diagonal components are above $5 \times 10^{-4}$. 
For BGD, due to the absence of the regularization term present in MESU, $\sigma_{n}$ value remains constrained to the initialization value, $\sigma_{init}=0.1$, which impedes the reduction of Hessian diagonal components.

In Fig.~\ref{fig:hessian_approx}d, we study the temporal evolution of the Hessian approximation of these two methods. We also evaluate MESU and BGD against two state-of-the-art Hessian approximation-based methods, elastic weight consolidation (EWC) \citep{kirkpatrick2017overcoming} and synaptic intelligence (SI) \citep{zenke2017continual}.
In MESU and BGD the quality of the Hessian approximation improves with an increasing number of iterations, $n$. For a lower $n$, the approximation is worse, as the hypothesis of locally constant curvature is not yet satisfied. However, as training progresses and the model converges to a region where the local curvature becomes constant, the approximation greatly improves. Our theoretical result holds as $n$ approaches infinity. At $n$=100,000 iterations, the Hessian approximation accuracy provided by MESU exceeds not only EWC and SI, but also that of BGD.

\end{paragraph}

\begin{figure}[h]
  \centering
  \includegraphics[width=0.90\textwidth]{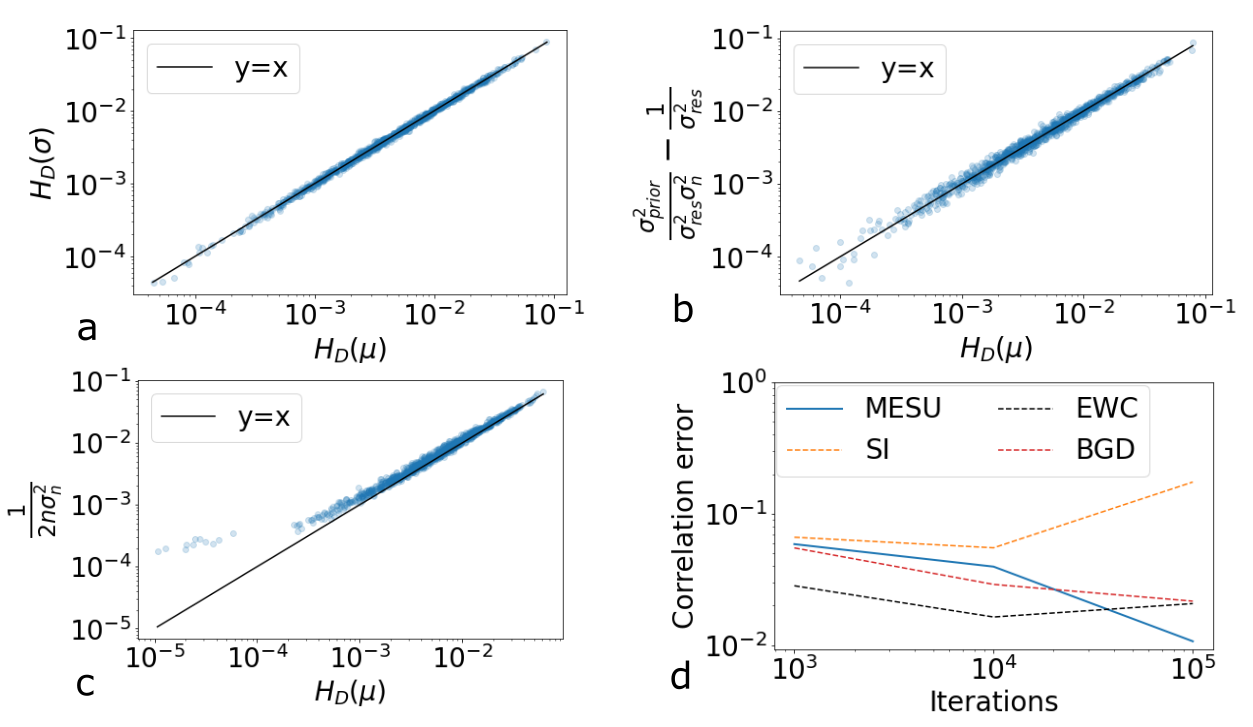}
  \caption{\textbf{Verification of theoretical results.}
  \textbf{a} Hessian diagonal with respect to $\mu$ vs. the Hessian diagonal with respect to $\sigma$ after $n$=100,000 iterations. \textbf{b} Approximation of the Hessian diagonal  by MESU (Eq.~(\ref{eq:HessianResult})) vs.  actual Hessian diagonal. \textbf{c} Approximation of the Hessian diagonal by MESU without regularization term vs. actual Hessian diagonal. \textbf{d} Correlation between the Hessian actual matrix and its approximation for EWC, SI, and MESU as a function of iteration updates.}
  \label{fig:hessian_approx}
\end{figure}

\begin{paragraph}{Vanishing plasticity}
One key advantage of MESU over BGD should be that, as the number of iterations $ n$ increases, the $\sigma$ parameters do not tend to zero - an effect we refer to as \textit{vanishing plasticity}: with MESU, the network should remain plastic (i.e., it can continually learn new tasks) due to the regularization term. To test this, we trained the neural network on 100 permutations of MNIST \citep{van2019three} using both MESU and BGD. The test accuracy in each of 100 tasks is evaluated after each epoch. Each permutation is learned over 20 epochs. Both models operate without any prior knowledge of task boundaries, complicating the preservation of previous knowledge. As shown in Fig.~\ref{fig:benchmark}a, the tasks evolve continuously, and the MESU method mitigates catastrophic forgetting. Fig.~\ref{fig:benchmark}b monitors the range of $\sigma$ values assumed by the models synapses as the learning process progresses for both methods. 
In the MESU method, the mean value of $\sigma$ remains relatively constant throughout the entire training process.
In contrast, for BGD, $\sigma$ decreases as the number of epochs increases. After 2000 epochs, corresponding to 100 tasks, the low $\sigma$ values will impede the network from learning new tasks effectively, indicating the vanishing plasticity problem.
This difference has important consequences on the learning process. We assessed continual learning at two key stages: the initial phase (the first ten tasks, Fig.~\ref{fig:benchmark}c) and the later phase (the last ten tasks, Fig.~\ref{fig:benchmark}d).
In the initial phase, BGD demonstrates a slight advantage, displaying less forgetting of the initial ten tasks. However, after 100 tasks, BGD experiences a significant drop in accuracy of around 6\%, while  MESU retains its ability to learn, almost equivalently to the initial phase.

\end{paragraph}

\begin{figure}[h]
  \centering
  \includegraphics[width=0.95\textwidth]{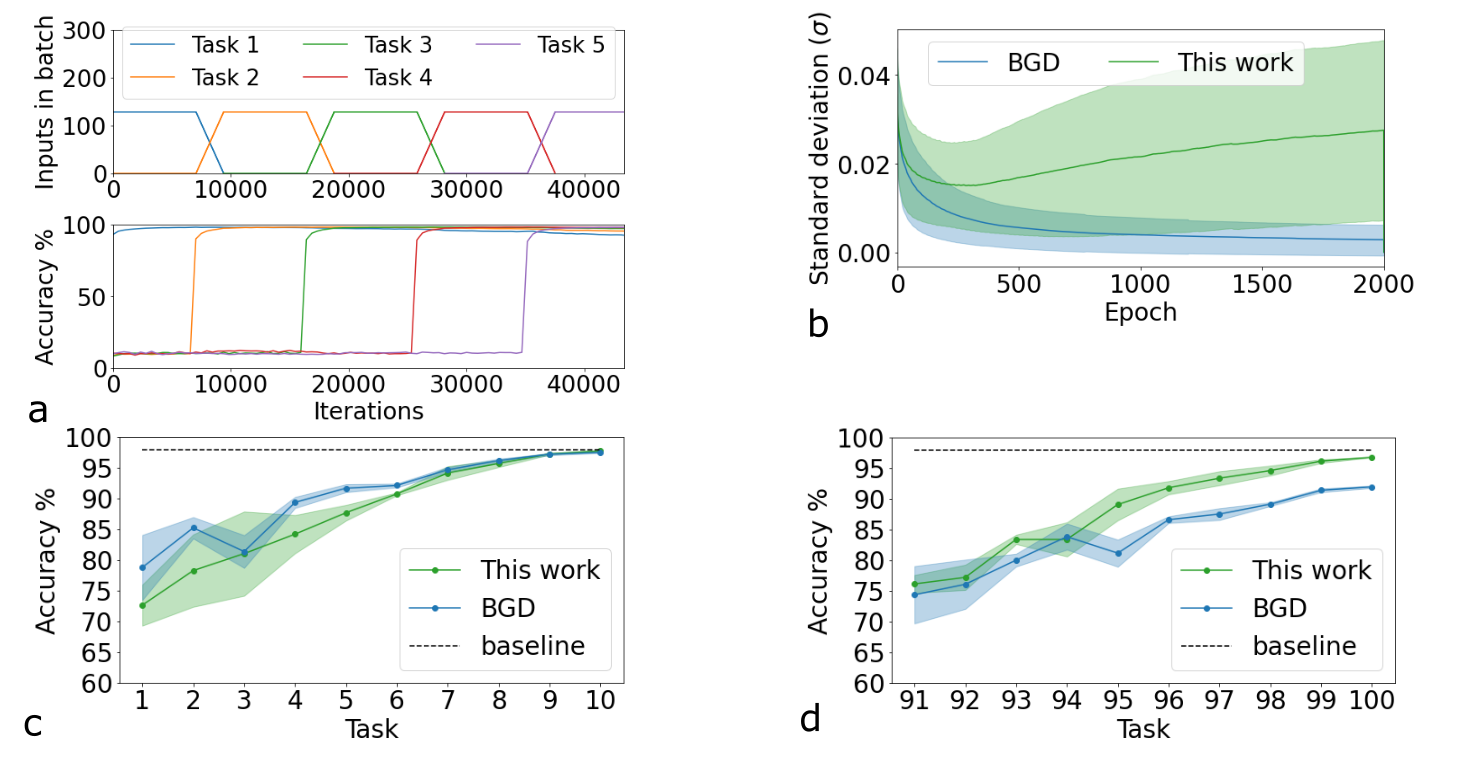}
  \caption{\textbf{Benchmark between BGD and MESU}. \textbf{a} Evolution of the number of input belonging to each task (different permutation of MNIST) present in a training batch as a function of the number of iterations, along with the evolution of MESU test accuracy corresponding to each task. \textbf{b} Evolution of the mean value of the synaptic uncertainty with respect to the number of epochs for BGD and MESU.  \textbf{c,d} Test accuracy for the last ten tasks learned after learning ten tasks (\textbf{c}) and 100 tasks (\textbf{d}) for BGD and MESU, averaged over 3 trials. Error bars represent one standard deviation. } 
  \label{fig:benchmark}
\end{figure}
\FloatBarrier

\section{Discussion}\label{sec4}

This paper presents three contributions: (1) the establishment of a link between synaptic uncertainty and the Hessian diagonal, (2) the introduction of a novel update rule, MEtaplasticity from Synaptic Uncertainty (MESU), which allows continuous learning without vanishing plasticity, and (3) a theoretical and empirical validation of these concepts. Our work also offers a novel interpretation of the previously proposed Bayesian Gradient Descent (BGD), enriching its theoretical underpinnings.

Our method is not alone in attempting continual learning without task boundaries.  \citet{zeno2018task} proposed task-boundary-less versions of  Online Elastic Weight Consolidation \citep{schwarz2018progress} and Memory-Aware Synapses \citep{aljundi2018memory}. While these methods show promise, previous work showed that BGD outperforms them in the absence of task boundaries \citep{zeno2018task}. However, all these techniques, including BGD, deteriorate in performance as the number of tasks increases. Our method not only largely matches BGD's performance for a limited number of initial tasks, but also maintains its learning capability over a massively extended range of tasks where BGD suffers from vanishing plasticity. 

Replay-based methods are another avenue for addressing catastrophic forgetting. They require sophisticated sample selection when task boundaries are absent \citep{aljundi2019gradient}. Still, an exciting future direction would be to investigate the potential synergies \citep{nguyen2017variational} between replay-based methods and MESU to achieve even better performance.

Our method inherits the computational challenges associated with Bayesian Neural Networks, primarily the expensive sampling process. While this is a limitation, recent work on hardware implementations of Bayesian Neural Networks, using e.g., memristors, suggests a promising avenue \citep{dalgaty2021ex,lin2023uncertainty,bonnet2023bringing}. Our work shows a unique convergence between Hessian computation, biological insights, and Bayesian principles, which is not only informative, but offers a deeper understanding of how to solve the stability-plasticity dilemma \citep{mermillod2013stability} in long-term continual learning scenarios.

\FloatBarrier

\bibliography{main}

\begin{thebibliography}{24}
\providecommand{\natexlab}[1]{#1}
\providecommand{\url}[1]{\texttt{#1}}
\expandafter\ifx\csname urlstyle\endcsname\relax
  \providecommand{\doi}[1]{doi: #1}\else
  \providecommand{\doi}{doi: \begingroup \urlstyle{rm}\Url}\fi

\bibitem[Abdar et~al.(2021)Abdar, Pourpanah, Hussain, Rezazadegan, Liu, Ghavamzadeh, Fieguth, Cao, Khosravi, Acharya, et~al.]{abdar2021review}
Moloud Abdar, Farhad Pourpanah, Sadiq Hussain, Dana Rezazadegan, Li~Liu, Mohammad Ghavamzadeh, Paul Fieguth, Xiaochun Cao, Abbas Khosravi, U~Rajendra Acharya, et~al.
\newblock A review of uncertainty quantification in deep learning: Techniques, applications and challenges.
\newblock \emph{Information fusion}, 76:\penalty0 243--297, 2021.

\bibitem[Aitchison et~al.(2021)Aitchison, Jegminat, Menendez, Pfister, Pouget, and Latham]{aitchison2021synaptic}
Laurence Aitchison, Jannes Jegminat, Jorge~Aurelio Menendez, Jean-Pascal Pfister, Alexandre Pouget, and Peter~E Latham.
\newblock Synaptic plasticity as bayesian inference.
\newblock \emph{Nature neuroscience}, 24\penalty0 (4):\penalty0 565--571, 2021.

\bibitem[Aljundi et~al.(2018)Aljundi, Babiloni, Elhoseiny, Rohrbach, and Tuytelaars]{aljundi2018memory}
Rahaf Aljundi, Francesca Babiloni, Mohamed Elhoseiny, Marcus Rohrbach, and Tinne Tuytelaars.
\newblock Memory aware synapses: Learning what (not) to forget.
\newblock In \emph{Proceedings of the European conference on computer vision (ECCV)}, pages 139--154, 2018.

\bibitem[Aljundi et~al.(2019)Aljundi, Lin, Goujaud, and Bengio]{aljundi2019gradient}
Rahaf Aljundi, Min Lin, Baptiste Goujaud, and Yoshua Bengio.
\newblock Gradient based sample selection for online continual learning.
\newblock \emph{Advances in neural information processing systems}, 32, 2019.

\bibitem[Becker et~al.(1988)Becker, Le~Cun, et~al.]{becker1988improving}
Sue Becker, Yann Le~Cun, et~al.
\newblock Improving the convergence of back-propagation learning with second order methods.
\newblock In \emph{Proceedings of the 1988 connectionist models summer school}, pages 29--37, 1988.

\bibitem[Blundell et~al.(2015)Blundell, Cornebise, Kavukcuoglu, and Wierstra]{blundell2015weight}
Charles Blundell, Julien Cornebise, Koray Kavukcuoglu, and Daan Wierstra.
\newblock Weight uncertainty in neural network.
\newblock In \emph{International conference on machine learning}, pages 1613--1622. PMLR, 2015.

\bibitem[Bonnet et~al.(2023)Bonnet, Hirtzlin, Majumdar, Dalgaty, Esmanhotto, Meli, Castellani, Martin, Nodin, Bourgeois, et~al.]{bonnet2023bringing}
Djohan Bonnet, Tifenn Hirtzlin, Atreya Majumdar, Thomas Dalgaty, Eduardo Esmanhotto, Valentina Meli, Niccol{\`o} Castellani, Simon Martin, Jean-Francois Nodin, Guillaume Bourgeois, et~al.
\newblock Bringing uncertainty quantification to the extreme-edge with memristor-based bayesian neural networks.
\newblock 2023.

\bibitem[D'Agostino et~al.(2023)D'Agostino, Moro, Hirtzlin, Arcamone, Castellani, Querlioz, Payvand, and Vianello]{d2023synaptic}
Simone D'Agostino, Filippo Moro, Tifenn Hirtzlin, Julien Arcamone, Niccol{\`o} Castellani, Damien Querlioz, Melika Payvand, and Elisa Vianello.
\newblock Synaptic metaplasticity with multi-level memristive devices.
\newblock \emph{arXiv preprint arXiv:2306.12142}, 2023.

\bibitem[Dalgaty et~al.(2021)Dalgaty, Esmanhotto, Castellani, Querlioz, and Vianello]{dalgaty2021ex}
Thomas Dalgaty, Eduardo Esmanhotto, Niccolo Castellani, Damien Querlioz, and Elisa Vianello.
\newblock Ex situ transfer of bayesian neural networks to resistive memory-based inference hardware.
\newblock \emph{Advanced Intelligent Systems}, 3\penalty0 (8):\penalty0 2000103, 2021.

\bibitem[Fusi et~al.(2005)Fusi, Drew, and Abbott]{fusi2005cascade}
Stefano Fusi, Patrick~J Drew, and Larry~F Abbott.
\newblock Cascade models of synaptically stored memories.
\newblock \emph{Neuron}, 45\penalty0 (4):\penalty0 599--611, 2005.

\bibitem[Gal(2016)]{gal2016uncertainty}
Yarin Gal.
\newblock Uncertainty in deep learning.
\newblock \emph{PhD thesis, University of Cambridge}, 2016.

\bibitem[Goodfellow et~al.(2013)Goodfellow, Mirza, Xiao, Courville, and Bengio]{goodfellow2013empirical}
Ian~J Goodfellow, Mehdi Mirza, Da~Xiao, Aaron Courville, and Yoshua Bengio.
\newblock An empirical investigation of catastrophic forgetting in gradient-based neural networks.
\newblock \emph{arXiv preprint arXiv:1312.6211}, 2013.

\bibitem[Jedlicka et~al.(2022)Jedlicka, Tomko, Robins, and Abraham]{jedlicka2022contributions}
Peter Jedlicka, Matus Tomko, Anthony Robins, and Wickliffe~C Abraham.
\newblock Contributions by metaplasticity to solving the catastrophic forgetting problem.
\newblock \emph{Trends in Neurosciences}, 45\penalty0 (9):\penalty0 656--666, 2022.

\bibitem[Kirkpatrick et~al.(2017)Kirkpatrick, Pascanu, Rabinowitz, Veness, Desjardins, Rusu, Milan, Quan, Ramalho, Grabska-Barwinska, et~al.]{kirkpatrick2017overcoming}
James Kirkpatrick, Razvan Pascanu, Neil Rabinowitz, Joel Veness, Guillaume Desjardins, Andrei~A Rusu, Kieran Milan, John Quan, Tiago Ramalho, Agnieszka Grabska-Barwinska, et~al.
\newblock Overcoming catastrophic forgetting in neural networks.
\newblock \emph{Proceedings of the national academy of sciences}, 114\penalty0 (13):\penalty0 3521--3526, 2017.

\bibitem[Laborieux et~al.(2021)Laborieux, Ernoult, Hirtzlin, and Querlioz]{laborieux2021synaptic}
Axel Laborieux, Maxence Ernoult, Tifenn Hirtzlin, and Damien Querlioz.
\newblock Synaptic metaplasticity in binarized neural networks.
\newblock \emph{Nature communications}, 12\penalty0 (1):\penalty0 2549, 2021.

\bibitem[LeCun(1998)]{lecun1998mnist}
Yann LeCun.
\newblock The mnist database of handwritten digits.
\newblock \emph{http://yann. lecun. com/exdb/mnist/}, 1998.

\bibitem[Lin et~al.(2023)Lin, Zhang, Gao, Tang, Yao, Li, Huang, Liu, Zhou, Liu, et~al.]{lin2023uncertainty}
Yudeng Lin, Qingtian Zhang, Bin Gao, Jianshi Tang, Peng Yao, Chongxuan Li, Shiyu Huang, Zhengwu Liu, Ying Zhou, Yuyi Liu, et~al.
\newblock Uncertainty quantification via a memristor bayesian deep neural network for risk-sensitive reinforcement learning.
\newblock \emph{Nature Machine Intelligence}, pages 1--10, 2023.

\bibitem[MacKay(1992)]{mackay1992practical}
David~JC MacKay.
\newblock A practical bayesian framework for backpropagation networks.
\newblock \emph{Neural computation}, 4\penalty0 (3):\penalty0 448--472, 1992.

\bibitem[Mermillod et~al.(2013)Mermillod, Bugaiska, and Bonin]{mermillod2013stability}
Martial Mermillod, Aur{\'e}lia Bugaiska, and Patrick Bonin.
\newblock The stability-plasticity dilemma: Investigating the continuum from catastrophic forgetting to age-limited learning effects, 2013.

\bibitem[Nguyen et~al.(2017)Nguyen, Li, Bui, and Turner]{nguyen2017variational}
Cuong~V Nguyen, Yingzhen Li, Thang~D Bui, and Richard~E Turner.
\newblock Variational continual learning.
\newblock \emph{arXiv preprint arXiv:1710.10628}, 2017.

\bibitem[Schwarz et~al.(2018)Schwarz, Czarnecki, Luketina, Grabska-Barwinska, Teh, Pascanu, and Hadsell]{schwarz2018progress}
Jonathan Schwarz, Wojciech Czarnecki, Jelena Luketina, Agnieszka Grabska-Barwinska, Yee~Whye Teh, Razvan Pascanu, and Raia Hadsell.
\newblock Progress \& compress: A scalable framework for continual learning.
\newblock In \emph{International conference on machine learning}, pages 4528--4537. PMLR, 2018.

\bibitem[Van~de Ven and Tolias(2019)]{van2019three}
Gido~M Van~de Ven and Andreas~S Tolias.
\newblock Three scenarios for continual learning.
\newblock \emph{arXiv preprint arXiv:1904.07734}, 2019.

\bibitem[Zenke et~al.(2017)Zenke, Poole, and Ganguli]{zenke2017continual}
Friedemann Zenke, Ben Poole, and Surya Ganguli.
\newblock Continual learning through synaptic intelligence.
\newblock In \emph{International conference on machine learning}, pages 3987--3995. PMLR, 2017.

\bibitem[Zeno et~al.(2018)Zeno, Golan, Hoffer, and Soudry]{zeno2018task}
Chen Zeno, Itay Golan, Elad Hoffer, and Daniel Soudry.
\newblock Task agnostic continual learning using online variational bayes.
\newblock \emph{arXiv preprint arXiv:1803.10123}, 2018.

\end{thebibliography}

\FloatBarrier

\section*{Acknowledgements}
This work was supported by European Research Council consolidator
grant DIVERSE (reference: 101043854) and by EIC Pathfinder METASPIN grant (reference: 101098651). In addition, we thank L. Hutin, A.
Laborieux, for discussing various aspects of the article.

\section*{Appendices}

\subsection*{A: Lemma}
\setcounter{lemma}{0}
\begin{lemma}[Hessian diagonal via First-Order Derivative in Bayesian Neural Networks]

Consider a mean-field Gaussian, $q(\omega|\theta)$, describing a Bayesian neural network, where $\theta=(\mu,\sigma)$, and weight samples are defined as $\omega=\mu+\epsilon \cdot \sigma$, with $\epsilon \sim \mathcal{N}(0, 1)$. Let $\mathcal{L}$ be the negative log-likelihood and $\mathcal{C}$ the expectation of $\mathcal{L}$, defined as $\mathcal{C}=\mathbb{E}_{q(\omega|\theta)}[\mathcal{L}(\omega)]$. The diagonal of the Hessian matrix of the expectation with respect to $\mu$ is given by: $H_D(\mu)=\frac{1}{\sigma} \frac{\partial \mathcal{C}}{\partial \sigma}$.
\end{lemma}

\begin{proof}
By definition, the diagonal elements of the Hessian matrix with respect to $\mu$ are:
$H_D(\mu) = \frac{\partial^2 \mathcal{C}}{\partial \mu^2}.$
Given that:
$\frac{\partial \mathcal{C}}{\partial \mu} = \mathbb{E}_{\epsilon}\left[\frac{\partial \mathcal{L}(\omega)}{\partial \omega}\right],
$
we deduce:
$H_D(\mu) = \mathbb{E}_{\epsilon}\left[\frac{\partial^2 \mathcal{L}(\omega)}{\partial \omega^2}\right].$

Similarly, we have:
$\frac{\partial \mathcal{C}}{\partial \sigma} = \mathbb{E}_{\epsilon}\left[\frac{\partial \mathcal{L}(\omega)}{\partial \omega} \times \epsilon\right].$

Using Stein's lemma, we obtain the following relation:
$
\mathbb{E}_{\omega}\left[\frac{\partial \mathcal{L}(\omega)}{\partial \omega} \times (\omega -\mu)\right] = \sigma^2 \mathbb{E_{\omega}}\left[\frac{\partial^2 \mathcal{L}(\omega)}{\partial \omega^2}\right].$
Therefore,
$
\mathbb{E}_{\epsilon}\left[\frac{\partial \mathcal{L}(\omega)}{\partial \omega} \times \epsilon\right] = \sigma \mathbb{E}_{\epsilon}\left[\frac{\partial^2 \mathcal{L}(\omega)}{\partial \omega^2}\right].$

We recognize the left term as $\frac{\partial \mathcal{C}}{\partial \sigma}$ and conclude:
\begin{equation}\label{eq:Hessian_mu_grad_sigma}
\boxed{H_D(\mu) = \frac{1}{\sigma} \frac{\partial \mathcal{C}}{\partial \sigma}}.
\end{equation}
\end{proof}

\subsection*{B: Sequences Analysis}
\FloatBarrier

\begin{proposition}[Study of a sequence convergence for Theorem 1]
Consider the sequence $\alpha_{n+1}=\frac{\alpha_n(1-\alpha_n^2)}{1-\alpha_{lim}^2}$, with $0<\alpha_0<1$ and $0<\alpha_{lim}<\frac{1}{\sqrt{3}}$. Then, the sequence on $(\alpha_n)$ converges to:
\begin{equation}\label{eq:}
\lim_{n \to +\infty} \alpha_n = \alpha_{lim}.
\end{equation}
\end{proposition}

\begin{proof}

\FloatBarrier

To demonstrate the convergence of this sequence we examine the two possible situations: $\alpha_0<\alpha_{lim}$ and $\alpha_0>\alpha_{lim}$.

$(i)$: Firstly, we consider $\alpha_0<\alpha_{lim}$.

By induction we establish that this upper bound stands $\forall n \in \mathbb{N}, \hspace{0.1cm} \alpha_n < \alpha_{lim}$.

\textbf{Base Case (Initialization)}: For n=0 

This holds true by the given assumptions.

\textbf{Inductive Step}: Assume for some arbitrary $n \in \mathbb{N}$ that : $\alpha_n < \alpha_{lim}$.

\begin{equation}\label{eq:}
\alpha_{n+1} -\alpha_{lim}= \frac{\alpha_n(1-\alpha_n^2)-\alpha_{lim}(1-\alpha_{lim}^2)}{1-\alpha_{lim}^2}.
\end{equation}

We define $f(x)=x(1-x^2)$. A simple analysis of the function gives $f$ increasing over the interval $[0,\frac{1}{\sqrt{3}}]$  and decreasing over the interval $[\frac{1}{\sqrt{3}},1]$ reaching its maximum value of $\frac{2}{3\sqrt{3}}$ at $\frac{1}{\sqrt{3}}$. The sequence is rewritten:

\begin{equation}\label{eq:}
\alpha_{n+1} -\alpha_{lim}= \frac{f(\alpha_n)-f(\alpha_{lim})}{1-\alpha_{lim}^2}.
\end{equation}

Given that the function $f$ is increasing over the interval $[0,\frac{1}{\sqrt{3}}]$ and that $0<\alpha_n<\alpha_{lim}<\frac{1}{\sqrt{3}}$, we can deduce that $f(\alpha_n)-f(\alpha_{lim}) <0$. Consequently: $\alpha_{n+1}<\alpha_{lim}$.

Therefore, 
\begin{equation}\label{eq:}
\forall n \in \mathbb{N}, \hspace{0.1cm} \alpha_n < \alpha_{lim}.
\end{equation}

According to the definition of the sequence $\alpha_{n+1}>\alpha_n$, hence the sequence $(\alpha_n)$ is monotonically increasing.
The sequence $(\alpha_n)$ is both increasing and upper bounded by $\alpha_{lim}$, $(\alpha_n)$ therefore, the function must converge. We denote the limit of this convergence as $l$.

When convergence is achieved the sequence becomes:  
$l= \frac{l(1-l^2)}{1-\alpha_{lim}^2}$
and we can deduce: $l=\alpha_{lim}$.

Finally:

\begin{equation}\label{eq:liù}
0<\alpha_0<\alpha_{lim}<\frac{1}{\sqrt{3}} \implies \lim_{n \to +\infty} \alpha_n = \alpha_{lim}.
\end{equation}

$(ii)$: Secondly, we consider $\alpha_{lim}<\alpha_0$.

If there exists an $n_0$ such that $\alpha_{n_0}<\alpha_{lim}$, then we can apply result of Eq. (\ref{eq:liù}) of the previous analysis. If no such $n_0$ exists, $\forall n \in \mathbb{N},$  $\alpha_{lim}<\alpha_n$,  we can immediately infer from the sequence's definition that  $\alpha_{n+1}<\alpha_n$.
The sequence $(\alpha_n)$ is lower bounded by $\alpha_{lim}$, which implies that the sequence $(\alpha_n)$ is monotonically decreasing. Therefore the sequence $(\alpha_n)$ converges and the limit is $\alpha_{lim}$, concluding the proof for any initial values satisfying $0<\alpha_0<1$ and $0<\alpha_{lim}<\frac{1}{\sqrt{3}}$:

\begin{equation}\label{eq:}
\boxed{\lim_{n \to +\infty} \alpha_n = \alpha_{lim}}.
\end{equation}

\end{proof}

\begin{proposition}[Study of a sequence's asymptotic behavior for Theorem 2]
Consider the sequence $\alpha_{n+1}=\alpha_n(1-\alpha_n^2)$ where $0<\alpha_0<1$. Then, the sequence on $(\alpha_n)$ converge to zero, and :
\begin{equation}\label{eq:}
\lim_{n \to +\infty} \alpha_n \sqrt{2n} =1.
\end{equation}
\end{proposition}




$(i)$: By induction, we establish that: $\forall n \in \mathbb{N^*}$, $\alpha_n \leq\frac{1}{\sqrt{2n}}$.

\textbf{Base Case (Initialization)}: n=1,2

From proposition 1 we establish that $f(x)=x(1-x^2) \leq \frac{2}{3 \sqrt 3}$. 

Thus, $\alpha_1 \leq \frac{2}{3 \sqrt 3}\leq \frac{1}{\sqrt 2}$ and $\alpha_2 \leq \frac{2}{3 \sqrt 3} \leq \frac{1}{\sqrt 4}$, confirming the base case.

\textbf{Inductive Step}: We assume the proposition holds for $n \geq 2 \in \mathbb{N}$, i.e., $\alpha_n \leq\frac{1}{\sqrt{2n}}.$


As $f$ is monotonically increasing on $[0,\frac{1}{\sqrt{3}}]$, and from the inductive hypothesis $f(\alpha_{n}) \leq f(\frac{1}{\sqrt{2n}}) $ provided:
\begin{equation}\label{eq:ineq}
\alpha_{n+1} \leq \frac{1}{\sqrt{2n}}(1-\frac{1}{2n}).
\end{equation}

Using Taylor's expansion we have the following expression :\begin{equation}\label{eq:taylor}
\frac{1}{\sqrt{2n+2}}=\frac{1}{\sqrt{2n}}(1-\frac{1}{2n}+\frac{3}{8n^2}+R_3(\frac{1}{n})).  
\end{equation}

Where $R_3$ represents the remainder of the Taylor series and Taylor's inequality gives $\forall n \geq 2$, $\lvert {R_3(\frac{1}{n})} \rvert <\frac{3}{8n^2}$. Eq.(\ref{eq:ineq}) and Eq.(\ref{eq:taylor}) lead to $\alpha_{n+1} \leq \frac{1}{\sqrt{2n+2}}$


\begin{equation}\label{eq:res i}
\forall n \in \mathbb{N^*}, \alpha_{n} \leq \frac{1}{\sqrt{2n}}.
\end{equation}

$(ii)$: We consider a sequence $(V_n)$ defined as: $V_n=\frac{n}{2} -n^2\alpha_n^2$ and demonstrate that the sequence $(V_n)$ is monotonically increasing. Taking the difference between successive terms, we have:

\begin{equation}\label{eq:}
V_{n+1}-V_{n}= \frac{1}{2}-(n+1)^2\alpha_{n+1}^2+n^2\alpha_n^2.
\end{equation}

Using the result from  $(i)$ we ascertain that: $(n+1)^2\alpha_{n+1}^2 \leq \frac{1}{2}$.

From which we infer $V_{n+1}-V_{n} \geq n^2\alpha_n^2 \geq 0$

Involving that the sequence $(Vn)$ is increasing.There is two possibilities: First possibility, ($V_n$) converges and the limit is: 
\begin{equation}\label{eq:}
\lim_{n \to +\infty}{V_n} = c \in \mathbb{R}.
\end{equation}

In consequence :

\begin{equation}\label{eq:}
\lim_{n \to +\infty} \frac{V_n}{n} = \lim_{n \to +\infty} n (\frac{1}{2n} -\alpha_n^2) =0.
\end{equation}

Concluding the proof when ($V_n$) converge:

\begin{equation}\label{eq:}
\boxed{\lim_{n \to +\infty} \alpha_n \sqrt{2n} =1}.
\end{equation}

Second possibility:

\begin{equation}\label{eq:ii}
 \lim_{n \to +\infty}{V_n} = +\infty.
\end{equation}

$(iii)$: We consider the sequence $(P_n)$ defined by: $P_n=\alpha_n \sqrt{2n}$.

We demonstrate that there is a given $n_0 \in \mathbb{N^*}$ such that $\forall n \geq n_0$, $P_n \leq P_{n+1}$. Involving $(P_n)$  increasing beyond $n_0$.

Looking at the ratio between successive terms, we find:

\begin{equation}\label{eq:}
\frac{P_{n+1}}{P_n}=\frac{\alpha_n(1-\alpha_n^2)\sqrt{2n+2}}{\alpha_n\sqrt{2n}}=(1-\alpha_n^2)\sqrt{1+\frac{1}{n}}.
\end{equation}

Further expanding, we have:

\begin{equation}\label{eq:}
\frac{P_{n+1}}{P_n}=1-\alpha_n^2+\frac{1}{2n}+O(\frac{1}{n^2}).
\end{equation}

Using the results from $(i)$, we ascertain $\frac{1}{2n}-\alpha_n^2 \geq 0$. Additionally, by invoking $(ii)$, it becomes clear that as $n$ grows, the term $O(\frac{1}{n^2})$ become negligible relative to $(\frac{1}{2n}-\alpha_n^2)$. Consequently, we can deduce that there exists a particular $n_0 \in \mathbb{N^*}$ for which all $n \geq n_0$, $\frac{P_{n+1}}{P_n} \geq 1$.

Concluding to:

\begin{equation}\label{eq:iii}
\exists n_0 \in \mathbb{N^*}, \forall n \geq n_0, (P_n)\nearrow.
\end{equation}

From $(i)$ we know that $(P_n) \leq 1$  and from $(iii)$ we know that $(P_n)$ is monotically increasing and so $(P_n)$ converge: 

\begin{equation}\label{eq:limpn}
\lim_{n \to +\infty} P_n  =c.
\end{equation}

$(iv)$ We consider $(log(P_n))$ demonstrating the convergence of the series $\sum (\frac{1}{2n} -\alpha_n^2)$.

First, let us rewrite $(\alpha_n)$.

\begin{equation}\label{eq:}
 \alpha_{n+2} =\alpha_{n+1}(1-\alpha_{n+1}^2) =  \alpha_{n}(1-\alpha_{n}^2)(1-\alpha_{n+1}^2).
\end{equation}

From the above recursion relation, we can inductively deduce the general form:

\begin{equation}\label{eq:}
 \alpha_{n} =\prod_{k=1}^{n-1}(1-\alpha_{k}^2)\alpha_1.
\end{equation}

This allows us to express $(log(P_n))$ as: 

\begin{equation}\label{eq:}
 log(P_n) =\sum_{k=1}^{n-1}log((1-\alpha_{k}^2)) + log(\alpha_1) +log(\sqrt{2n}),
\end{equation}

which, in the limit as $n$ approaches infinity, is asymptotic to:

\begin{equation}\label{eq:}
 log(P_n) \hspace{0.3cm} \underset{n \to +\infty}{\sim} \hspace{0.3cm} \sum_{k=1}^{n-1}log((1-\alpha_{k}^2))   +\frac{1}{2}log(n).
\end{equation}

From a prior result $(i)$, we have:

\begin{equation}\label{eq:}
log((1-\alpha_{k}^2)) = -\alpha_k^2 +O(\frac{1}{n^2}).
\end{equation}

Invoking the properties of the Harmonic series, we can further deduce:

\begin{equation}\label{eq:}
 log(P_n) \hspace{0.3cm} \underset{n \to +\infty}{\sim} \hspace{0.3cm} \sum_{k=1}^{n-1}-\alpha_{k}^2 +  \frac{1}{2}\sum_{k=1}^{n-1} \frac{1}{k}.
\end{equation}

Given Equation (\ref{eq:limpn}): 
 
 \begin{equation}\label{eq:}
\lim_{n \to +\infty} log(P_n)  =log(c).
\end{equation}

In consequence the series $\sum (\frac{1}{2n} -\alpha_n^2)$ converge. 

Proposition 2 is demonstrated based on the previous arguments .Based on previous results, namely $(i)$ and $(iii)$, we know that $(\frac{1}{2n} -\alpha_n^2) \geq 0$ and the sequence $(\frac{1}{2n} -\alpha_n^2)$ is monotonically decreasing.

We consider the two sequences $(u_n)$ and $(v_n)$ such that $u_n=\frac{1}{2n} -\alpha_n^2$ and $v_n=n(u_{n}-u_{n+1})$.
By rearranging terms, the sum of the first $n$ terms of $v_k$ can be expressed as:

\begin{equation}\label{eq:}
\sum_{k=0}^{n} v_k=\sum_{k=1}^{n} ku_{k}-\sum_{k=1}^{n+1} (k-1)u_{k},
\end{equation}

which gives:

\begin{equation}\label{eq:}
\sum_{k=0}^{n} u_k=\sum_{k=0}^{n} v_k + nu_{n+1}.
\end{equation}

Given that $\sum u_k$ converges in $(iv)$, it follows that $(\sum_{k=0}^{n} v_k + nu_{n+1})$ converges too. Since $u_n \geq 0$  and the sequence $(u_n)$ is monotonically decreasing, $v_n \geq 0$. Hence, both of the series $\sum v_k$ and $nu_{n+1}$ converge.

Considering the limit of the product $nu_n$, we have:

\begin{equation}\label{eq:}
\lim_{n \to +\infty} n u_n = l.
\end{equation}

If $l>0$, we have $u_n \underset{n \to +\infty}{\sim} \frac{l}{n}$. However, this would lead to the divergence of $\sum u_k$ and consequently we deduce that $l=0$ and therefore:

\begin{equation}\label{eq:}
\lim_{n \to +\infty} n u_n =0,
\end{equation}

from which we can deduce:

 \begin{equation}\label{eq:}
\lim_{n \to +\infty} \frac{V_n}{n} = \lim_{n \to +\infty} n (\frac{1}{2n} -\alpha_n^2) =0.\end{equation}

Concluding the proof when ${V_n} \rightarrow \infty.$:

\begin{equation}\label{eq:}
\boxed{\lim_{n \to +\infty} \alpha_n \sqrt{2n} =1}
\end{equation}
\newpage

\subsection*{C: Theorems demonstration}

\setcounter{theorem}{0}
\begin{theorem}\label{thm:MESU}[Hessian Diagonal from Synaptic Uncertainty in Bayesian Neural Networks]
Let $q(\omega|\theta)$ represent the mean-field Gaussian describing the Bayesian neural network, where $\theta=(\mu,\sigma)$, and weight samples are defined as $\omega=\mu+\epsilon \cdot \sigma$, with $\epsilon \sim \mathcal{N}(0, 1)$. Consider $\mathcal{L}$ as the negative log-likelihood and $\mathcal{C}$ as the expectation of $\mathcal{L}$, defined as $\mathcal{C}=\mathbb{E}_{q(\omega|\theta)}[\mathcal{L}(\omega)]$. Let $\sigma_{n+1}=\sigma_n-\sigma_n^2 \times \frac{\partial \mathcal{C}}{\partial \sigma_n}+ \frac{\sigma_n(\sigma_{prior}^2 -\sigma_n^2)}{\sigma_{res}^2}$ be the update rule over $\sigma$. If the loss curvature can be expressed as $\frac{1}{\sigma_L^2}$,  $\sigma_0<\sigma_L$ and, $\frac{\sigma_{prior}^2} {\sigma_{res}^2}<<1$, then the diagonal elements of the Hessian matrix with respect to $\mu$ and $\sigma$ are given by:
\begin{equation}\label{eq:}
H_D(\mu)=H_D(\sigma)=\lim_{n \to +\infty} \left[\frac{\sigma_{prior}^2 }{\sigma_{res}^2\sigma_n^2}-\frac{1}{\sigma_{res}^2}\right].
\end{equation}
\end{theorem}

\begin{proof} First, let us define: $\alpha_{lim}=\frac{\sigma_{prior}}{\sigma_{res}}$, and $\alpha_n=\sigma_{n}\sqrt{\frac{1}{\sigma_{l}^2}+\frac{1}{\sigma_{res}^2}}$ and let us consider the sequence $(\alpha_n)$, defined as:

 \begin{equation}\label{eq:}
\alpha_{n+1}=\frac{\alpha_n(1-\alpha_n^2)}{1-\alpha_{lim}^2}.
\end{equation}

We will now use a numerical approximation to simplify this sequence. When $\alpha_{lim}^2<<1$, we have:

\begin{equation}\label{eq:}
\alpha_{n+1}=\alpha_n(1-\alpha_n^2)(1+\alpha_{lim}^2)
\end{equation}

\begin{equation}\label{eq:}
\sigma_{n+1}=\sigma_n(1-\frac{\sigma_n^2 }{\sigma_L^2}-\frac{\sigma_n^2 }{\sigma_{res}^2})(1+ \frac{\sigma_{prior}^2}{\sigma_{res}^2})
\end{equation}

\begin{equation}\label{eq:}
\sigma_{n+1}=\sigma_n-\frac{\sigma_n^3 }{\sigma_L^2}-\frac{\sigma_n^3 }{\sigma_{res}^2} + \frac{\sigma_n \sigma_{prior}^2}{\sigma_{res}^2} - \frac{\sigma_n^3 \sigma_{prior}^2}{\sigma_{res}^2\sigma_{l}^2} -\frac{\sigma_n^3 \sigma_{prior}^2}{\sigma_{res}^4}
\end{equation}

\begin{equation}\label{eq:}
\sigma_{n+1}=\sigma_n(1-\frac{\sigma_n^2}{\sigma_L^2}) + \frac{\sigma_n(\sigma_{prior}^2 -\sigma_n^2)}{\sigma_{res}^2} - \frac{\sigma_n^3 \sigma_{prior}^2}{\sigma_{res}^2\sigma_{L}^2} -\frac{\sigma_n^3 \sigma_{prior}^2}{\sigma_{res}^4}
\end{equation}

\begin{equation}\label{eq:}
\sigma_{n+1}=\sigma_n(1-\frac{\sigma_n^2}{\sigma_L^2}(1+\frac{\sigma_{prior}^2}{\sigma_{res}^2})) + \frac{\sigma_n(\sigma_{prior}^2 -\sigma_n^2(1+\frac{\sigma_{prior}^2}{\sigma_{res}^2}))}{\sigma_{res}^2}.
\end{equation}

Again, since $\alpha_{lim}^2<<1$, this can be further simplified as:

\begin{equation}\label{eq:}
\sigma_{n+1}=\sigma_n(1-\frac{\sigma_n^2}{\sigma_L^2}) + \frac{\sigma_n(\sigma_{prior}^2 -\sigma_n^2)}{\sigma_{res}^2}.
\end{equation}

We recognize the sequence we obtained in Eq. (\ref{eq:sequence}), which was derived from our update rule, the assumption of locally constant curvature and by invoking Lemma 1.

In Appendix B, Proposition 1, we demonstrate that: 

\begin{equation}\label{eq:}
\lim_{n \to +\infty} \alpha_n = \alpha_{lim}.
\end{equation}

By squaring the result and reconfiguring the expression, we obtain:

\begin{equation}\label{eq:}
\boxed{H_D(\mu)=H_D(\sigma)=\lim_{n \to +\infty} \left[\frac{\sigma_{prior}^2 }{\sigma_{res}^2\sigma_n^2}-\frac{1}{\sigma_{res}^2}\right]}.
\end{equation}

\end{proof}

\begin{theorem}[Hessian Diagonal via Synaptic Uncertainty Asymptote in Bayesian Neural Networks]
Let $q(\omega|\theta)$ represent the mean-field Gaussian describing the Bayesian neural network, where $\theta=(\mu,\sigma)$ and weight samples are defined as $\omega=\mu+\epsilon \cdot \sigma$, with $\epsilon \sim \mathcal{N}(0, 1)$. Consider $\mathcal{L}$ as the negative log-likelihood and $\mathcal{C}$ as the expectation of $\mathcal{L}$, defined as $\mathcal{C}=\mathbb{E}_{q(\omega|\theta)}[\mathcal{L}(\omega)]$. Let $\sigma_{n+1}=\sigma_n-\sigma_n^2\times \frac{\partial \mathcal{C}}{\partial \sigma_n}$ be the update rule over $\sigma$. If the loss curvature can be expressed as $\frac{1}{\sigma_L^2}$ and $\sigma_0<\sigma_L$, then the diagonal elements of the Hessian matrix with respect to $\mu$ and $\sigma$ are given by:
\begin{equation}
 H_D(\mu)=H_D(\sigma)=\lim_{{n \to +\infty}} \frac{1}{2n\sigma_n^2}.  
\end{equation}
\end{theorem}

\begin{proof}
Thanks to the Lemma 1 and the locally constant curvature hypothesis, the update rule over $\sigma$ becomes: $\sigma_{n+1}=\sigma_n(1-\frac{\sigma_n^2}{\sigma_L^2})$.
By introducing $\alpha_n=\frac{\sigma_n}{\sigma_L}$, we establish the following recurrent sequence:

 \begin{equation}\label{eq:}
\alpha_{n+1}=\alpha_n(1-\alpha_n^2)
\hspace{3cm}
0<\alpha_0<1
\end{equation}

In Appendix B, Proposition 2, we demonstrate:

\begin{equation}\label{eq:}
\lim_{n \to +\infty} \alpha_n \sqrt{2n} =1,
\end{equation}

using $\alpha_n=\frac{\sigma_n}{\sigma_L}$ and $H_D(\mu)=\frac{1}{\sigma_L^2}$. Taking the square of the result we obtain: 

\begin{equation}
\boxed{H_D(\mu)=H_D(\sigma)=\lim_{{n \to +\infty}} \frac{1}{2n\sigma_n^2}}.  
\end{equation}

\end{proof}

\subsection*{D: Experimental Implementation}

\begin{paragraph}{Verification of theoretical result}
The model was trained on MNIST where the input images were scaled between 0 and 1. We employed a fully connected neural network with one hidden layer of 100 neurons, using ReLU activation function. 
The model was trained over 100,000 mini-batches, where the batch size was 100.
The mean values of the Bayesian Neural Network parameter distributions were initialized to the weights obtained from training 
an equivalent deterministic model - 
aligning with the hypothesis of locally constant curvature even at $n = 0$. We used the following hyperparameters: $\sigma_0=0.1$, $\sigma_{prior}= 0.1$, and $\sigma_{res}= 10$. 
\end{paragraph}

\begin{paragraph}{Comparison with the Fisher matrix and SI approximation of the Hessian}
The model was trained on MNIST task in the same fashion as for the verification of the theoretical results, 
besides the fact the mean values of the Bayesian neural network 
and the weights of the deterministic neural network were initialized via Kaiming uniform initialization. We used the following hyperparameters, $\sigma_0=0.06$, $\sigma_{prior}= 0.06$, and $\sigma_{res}= 10$. For the deterministic neural network, we used a stochastic gradient descent optimizer with $lr$=0.01.    
\end{paragraph}

\begin{paragraph}{Benchmark between MESU and BGD}
The MESU and BGD models were trained on permuted MNIST and the input images were padded to a dimension of 32x32 and then scaled between 0 and 1. We used a fully connected neural network with two hidden layers of 200 neurons, and used ReLU activation function. The batch size was 128. The mean values of the Bayesian neural networks were initialized using Kaiming uniform initialization. We used the following hyperparameters, $\sigma_0=0.04$, $\sigma_{prior}= 0.04$, and $\sigma_{res}= 10$.    
\end{paragraph}

\end{document}